\documentclass[letterpaper]{article}
\usepackage{aaai2027}  
\usepackage[hyphens]{url}  
\usepackage{graphicx} 
\usepackage{natbib}  
\usepackage{caption} 
\usepackage{amsmath,amssymb}
\usepackage{amsthm}
\newtheorem{proposition}{Proposition}
\newtheorem{assumption}{Assumption}
\usepackage{algorithm}
\usepackage{algorithmic}

\usepackage{newfloat}
\usepackage{listings}
\DeclareCaptionStyle{ruled}{labelfont=normalfont,labelsep=colon,strut=off} 
\floatstyle{ruled}
\newfloat{listing}{tb}{lst}{}
\floatname{listing}{Listing}

\usepackage{booktabs}

\title{SSPO: Structure-Aware Similarity-Weighted Preference Optimization for Neural Combinatorial Optimization}
\author{
    \begin{tabular}{c}
    Yuanyu Li\textsuperscript{\rm 1}\equalcontrib,
    Jintao Xu\textsuperscript{\rm 1}\equalcontrib,
    Zijiang Liu\textsuperscript{\rm 1}\equalcontrib,
    Yongzhi Qi\textsuperscript{\rm 1}\thanks{Corresponding authors.},\\
    Ningxuan Kang\textsuperscript{\rm 1},
    Jianshen Zhang\textsuperscript{\rm 1},
    Wei Qi\textsuperscript{\rm 2}\textsuperscript{\rm \dag},
    Chen Xie\textsuperscript{\rm 1},
    Zuo-Jun Max Shen\textsuperscript{\rm 3}
    \end{tabular}
}
\affiliations{

\textsuperscript{\rm 1}Supply Chain Tech Team Y, JD.com, Beijing, China\\
\textsuperscript{\rm 2}Department of Industrial Engineering, Tsinghua University, Beijing, China\\
\textsuperscript{\rm 3}Faculty of Engineering and Faculty of Business and Economics, The University of Hong Kong, Hong Kong, China\\
    \{liyuanyu6, xujintao.3014, liuzijiang3, qiyongzhi1, kangningxuan, zhangjianshen, xiechen\}@jd.com,\\
    qiw@tsinghua.edu.cn, maxshen@hku.hk

}

\begin{document}

\nocopyright
\maketitle

\begin{abstract}
\begin{quote}
Neural combinatorial optimization (NCO) relies on parallel solution
sampling for training, yet existing methods fail to fully exploit
the rich information latent in a co-sampled solution group.
Preference-optimization methods anchor on the single
best solution and discard fine-grained quality and structural signal
from all other peers---a failure we term \emph{gradient signal
polarization}.
Mean-based baselines instead weight peers
uniformly, so structurally near-identical peers flood the baseline
with redundant information and keep gradient variance high---a
failure we term \emph{baseline redundancy}.
We propose \textbf{SSPO} (Structure-Aware Similarity-Weighted
Preference Optimization), which scores all $B$ sampled solutions
jointly through a dissimilarity-weighted leave-one-out baseline:
structurally distinct peers receive higher weight, resolving both
failures in a single mechanism. The baseline uses zero-parameter,
problem-adaptive solution embeddings built from the encoder's
existing node representations.
Experiments on TSP, EFL, and JSP benchmarks show consistent gains over prior best-anchor and uniform-weight baselines. A direct comparison against uniform RLOO on TSP and EFL confirms that structure-aware weighting is  the primary driver of improvement.
The SSPO-trained EFL policy has been deployed in
a production facility-location system at JD$\mathord{.}$com, confirming practical viability at scale.

\end{quote}
\end{abstract}

\section{Introduction}

Combinatorial optimization problems (COPs)---such as the Traveling
Salesman Problem (TSP) \citep{dantzig1954tsp}, Job-Shop Scheduling (JSP) \citep{Helga2018jsp}, and the
Entity Facility Location (EFL) problem---are key to logistics,
scheduling, and resource allocation.
Traditional exact solvers (e.g., SCIP\footnote{https://www.scipopt.org/}, Gurobi\footnote{https://www.gurobi.com/})
guarantee optimality but scale poorly to large instances.
Neural combinatorial optimization
(NCO)~\cite{bello2016neural,kool2019attention} learns
solution-construction policies directly from data, achieving near-optimal
quality with orders-of-magnitude faster inference.
A hallmark of modern NCO training is \emph{parallel solution sampling}:
for each instance, $B$ candidate solutions are drawn from the current
policy simultaneously, and their collective cost signal is used to
update the policy parameters.
How to maximally exploit the information contained in this co-sampled
group is therefore central to training efficiency---yet existing methods
leave substantial signal on the table through two complementary failure
modes.

\paragraph{Failure Mode 1: Gradient Signal Polarization.}
Self-labeling methods and preference-optimization
approaches such as BOPO~\cite{liao2025bopo} treat the single best solution in a group as a
privileged ``anchor'' and cast every other sample as an undifferentiated
negative.
All training signal is thus concentrated at one extreme of the quality
spectrum.
The fine-grained quality gradients that exist \emph{among} the
non-best solutions---and the structural differences that explain why
one non-best solution is better than another---are completely
discarded.
When $B$ is large (e.g., $B=256$), the wasted information is immense.
This is a specific instance, in the NCO parallel-sampling regime, of
the well-known best-vs-rest information-loss problem studied in
pairwise preference learning and RLHF; what is particular to NCO is
that the discarded ``losers'' carry rich \emph{structural} signal
(which edges were chosen, which facilities were selected) that no
existing method has exploited.

\paragraph{Failure Mode 2: Baseline Redundancy and High Variance.}
Another family of methods avoids the best-anchor bias by incorporating
all $B$ solutions into the baseline via their average cost---whether
as a simple group mean or a leave-one-out variant---but weights them
\emph{uniformly} regardless of their structural relationships.
For NCO problems, co-sampled solutions often share strong topological
similarity, especially early in training when the policy is
near-deterministic.
This uniform weighting floods the baseline with redundant information
from near-identical peers, yielding a less discriminative control variate and leaving gradient variance persistently high.
The general phenomenon---that correlated samples weaken the
variance-reduction power of a control variate---is classical in Monte
Carlo estimation; our contribution is to show that in NCO it can be
diagnosed and \emph{repaired} cheaply using structural information
that is already present in the encoder's hidden states.

\subsection{Our Approach: SSPO}
Both failure modes share a common root: existing methods are
\emph{structure-blind}, either ignoring the structure of non-best
solutions entirely (polarization) or equating structurally identical and diverse solutions (redundancy).

We propose \textbf{SSPO} (Structure-Aware Similarity-Weighted
Preference Optimization) to address both problems simultaneously with
a single, unified mechanism.
The key insight is that a peer solution which is \emph{structurally
different} from the current solution represents a genuinely alternative
strategy: its cost provides far richer information about the policy
landscape than a near-duplicate.
SSPO therefore:
\begin{itemize}
  \item \textbf{Scores all $\boldsymbol{B}$ solutions jointly}---beyond simple best-vs-rest comparisons---preserving fine-grained quality differences among non-best solutions in the gradient signal.
  \item \textbf{Weights each peer by structural dissimilarity}, assigning
    higher baseline weight to topologically distinct solutions, so that informative contrast is amplified
    and redundant similarity is suppressed.
\end{itemize}
Concretely, SSPO builds \emph{zero-parameter} solution embeddings from the encoder's existing node representations: Hadamard products of node pairs for edge-structured problems such as TSP, and mean-pooling over solution-relevant selected nodes for problems whose solutions are represented by selected or allocated node sets, such as EFL and JSP. It then centers these embeddings to remove shared global bias and derives dissimilarity weights $w_{ij} \propto (1 - S_{ij})$, where $S_{ij}$ denotes the cosine similarity between solutions $i$ and $j$. Since $S_{ii}=1$, the self-weight $w_{ii}=0$ follows automatically, excluding the reference cost from its own baseline and preserving the leave-one-out construction.
When all solutions are structurally identical, $w_{ij}$ degenerates
uniformly to $1/(B-1)$, recovering standard RLOO as a special case.

\subsection{Contributions}

\begin{enumerate}
  \item Current NCO training suffers from two failure modes---%
  \emph{gradient signal polarization} and \emph{baseline redundancy}---%
    both caused by structure-blindness. We provide a formal diagnosis of each.
  \item We propose \textbf{SSPO}, a structure-aware, dissimilarity-weighted
    leave-one-out baseline that simultaneously addresses both failure
    modes: it uses \emph{all} $B$ solutions (no polarization) and
    up-weights structurally dissimilar peers (no redundancy). Structural
    similarity is measured via \emph{problem-adaptive},
    zero-parameter solution embeddings---Hadamard node-pair products
    for TSP and mean-pooling over selected nodes for EFL and JSP---%
    that fully reuse the encoder's existing representations, so SSPO retains the leave-one-out construction and introduces no additional learnable parameters.
  \item Experiments on TSP, EFL, and JSP benchmarks show consistent
    gains over prior best-anchor and uniform-weight baselines. A
    direct comparison against uniform RLOO on TSP and EFL isolates
    structure-aware weighting as the primary driver of improvement.
\end{enumerate}

\section{Related Work}

\subsection{Neural Combinatorial Optimization}

Pointer Networks~\cite{PointNetwork} pioneered autoregressive
construction policies for COPs under supervised learning.
The Attention Model (AM)~\cite{kool2019attention} established the
transformer-based encoder--decoder paradigm with REINFORCE training.
POMO~\cite{kwon2020pomo} amplifies solution diversity at inference by
using $B$ different start nodes, and adopts the group mean cost as a
shared baseline; however, this baseline is biased and still aggregates
all $B$ solutions with uniform weight.
Graph convolutional networks~\cite{joshi2019efficient} and
neural--heuristic hybrids~\cite{xin2021neurolkh} provide complementary
approaches.
None of the above exploit the \emph{structural relationships} among
co-sampled solutions as a training signal.

\subsection{Gradient Signal Polarization: Best-Anchor Methods}

A prominent line of work treats the best solution in a sampled group
as a privileged supervisory signal.
Self-labeling / self-improvement methods use the
best candidate as a pseudo-label and train the policy to imitate it,
discarding all other samples.
Preference-optimization approaches~\cite{pan2025preference,liao2025bopo} form pairwise
winner--loser pairs where the best solution is always the
``winner'' and every other solution is an undifferentiated ``loser.''
Both paradigms \emph{polarize} the training signal: the entire
learning signal collapses onto a single best-anchor solution, while
the fine-grained quality gradients and structural variation among the
remaining $B{-}1$ solutions are wholly wasted.
As $B$ grows, the discarded information becomes increasingly
significant.
SSPO avoids polarization entirely by assigning a meaningful,
differentiated weight to \emph{every} solution in the group.

\subsection{Baseline Redundancy: Uniform-Weighting Methods}

Methods that aggregate all $B$ solutions via a mean-based baseline
avoid the best-anchor bias, but weight every peer solution
\emph{equally} regardless of structural relationships.
When co-sampled solutions are topologically near-identical---a common
occurrence early in training or with greedy decoders---the
uniformly-weighted baseline accumulates redundant information from
structurally equivalent peers.
The resulting control variate has low discriminative power, and
gradient variance remains persistently high.
SSPO breaks the uniformity assumption while retaining the
leave-one-out construction: structurally dissimilar peers
receive higher weight, while the reference cost is excluded
from its own baseline.

\subsection{Leave-One-Out Baselines and Group-Relative Methods}
\label{sec:rw_rloo}

Leave-one-out (LOO) baselines have a long history as low-variance,
unbiased control variates for REINFORCE.
Kool et al.~\cite{kool2019buy} introduced the
\emph{RLOO} estimator, in which the baseline for sample $i$ is the
\emph{uniform} mean of the costs of the remaining $B{-}1$ peers.
Ahmadian et al.~\cite{ahmadian2024back} recently rediscovered RLOO in
the LLM RLHF setting and showed it is a competitive, simpler
alternative to PPO for language-model alignment.
\emph{Group Relative Policy Optimization}
(GRPO)~\cite{shao2024deepseekmath}, popularized by DeepSeek-Math/-R1~\citep{guodeepseekr12025},
follows the same template: it groups $B$ rollouts per prompt and
forms each sample's advantage relative to the \emph{uniform} mean and
variance of its peers.
POMO's group-mean baseline~\cite{kwon2020pomo}, while biased (it
includes the sample's own cost), is structurally identical: all peers
contribute with equal weight.

These methods all instantiate a common template,
\(b_i = \sum_{j \neq i} w_{ij}\,c_j\) with
\(w_{ij} = 1/(B{-}1)\),
and their effectiveness as control variates degrades whenever peers
are strongly correlated---precisely the regime NCO operates in, where
co-sampled solutions are often topologically near-identical.
SSPO can be seen as a structure-dependent generalization of this
template: it keeps the leave-one-out form but replaces the uniform weight with $w_{ij} \propto (1-S_{ij})$, where $S_{ij}$ is computed from
solution structure rather than cost.
To the best of our knowledge, SSPO is the first method to make leave-one-out weights \emph{structure-dependent} in the NCO setting.

\subsection{Structural Similarity as a Training Signal}

Structural diversity metrics (Hamming distance, Jaccard edge overlap)
are standard in population-based metaheuristics~\cite{helsgaun2017extension}
for maintaining solution variety, but are used at \emph{inference} or
\emph{search} time, not during policy gradient training. SSPO instead brings this structural signal into policy gradient training, using it \emph{online} to reweight the baseline for variance reduction.

\section{Problem Formulation}

\subsection{Graph-Based Optimization Instance}

We represent a COP instance as a graph
$\mathcal{G} = (V, E, \mathbf{X}, \mathbf{A})$,
where $V$ is the set of $N$ decision nodes,
$\mathbf{X} \in \mathbb{R}^{N \times d_x}$ are node features
(e.g., coordinates, demands), and
$\mathbf{A} \in \mathbb{R}^{N \times N \times d_a}$ are edge features
(e.g., pairwise distances).

\subsection{Learning Objective}
A policy $\pi_\theta$ autoregressively constructs a solution
$\mathbf{s} = (s_0, s_1, \ldots, s_{T-1})$ by selecting actions from
the feasible set at each step.
The training objective is to minimize the expected solution cost:
\begin{equation}
  \mathcal{J}(\theta) = \mathbb{E}_{\mathbf{s} \sim \pi_\theta}
  \bigl[c(\mathbf{s})\bigr],
\end{equation}
where $c(\mathbf{s})$ is a problem-specific cost.
For minimization problems $c(\mathbf{s})$ is the objective directly
(e.g., total tour length for TSP, makespan for JSP);
for maximization problems it is the negated objective
(e.g., $c(\mathbf{s}) = -\,$satisfied demand for EFL, whose goal is to
maximize demand coverage).
Throughout, we adopt the cost-minimization convention without loss of
generality.

\section{Structure-Aware Similarity-Weighted\\ Preference Optimization}

\subsection{Overview}

\begin{figure*}[t]
\centering
\includegraphics[width=\linewidth]{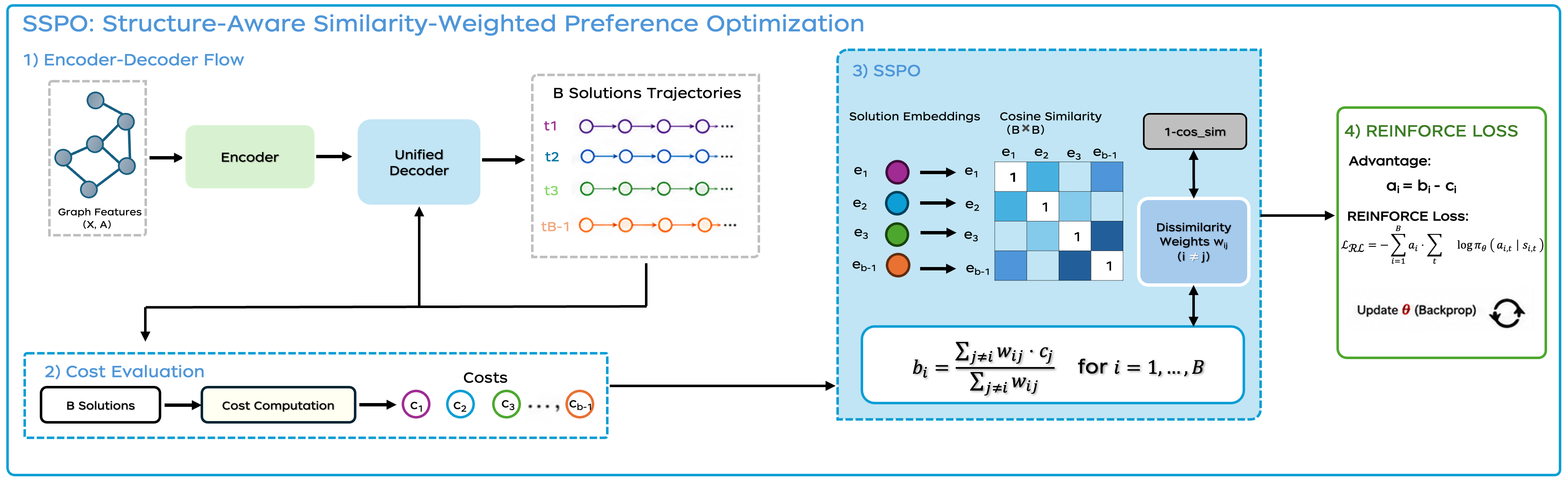}
\caption{Overview of SSPO.
For each instance, $B$ solutions are sampled from the policy.
Zero-parameter solution embeddings are built from the encoder's node
representations via Hadamard products (TSP) or mean-pooling (EFL/JSP),
then centered and used to compute pairwise dissimilarity weights.
The dissimilarity-weighted leave-one-out baseline is used in place of
a uniform baseline in the REINFORCE update.}
\label{fig:overview}
\end{figure*}

For each training instance, SSPO samples $B$ solutions
$\{\mathbf{s}_i\}_{i=1}^B$ with costs $\{c_i\}_{i=1}^B$
from the current policy $\pi_\theta$.
Unlike best-anchor methods that discard all but the optimal sample,
and unlike uniform mean-based baselines that ignore structural diversity,
SSPO constructs a \emph{dissimilarity-weighted leave-one-out} baseline
$b_i^{\text{SSPO}}$ that assigns meaningful, differentiated weights
to \emph{all} $B$ solutions based on their pairwise structural
dissimilarity.
The design rests on two components (Figure~\ref{fig:overview}).
First, we represent each solution by a structure-aware embedding.
It is built directly from the encoder's node representations and adds no parameters.
Second, we turn the pairwise dissimilarities between these embeddings
into leave-one-out baseline weights, so that a structurally distinct
peer contributes more to the baseline than a near-duplicate.
The resulting advantage $A_i = b_i^{\text{SSPO}} - c_i$ then drives a
standard REINFORCE update.

\subsection{Structure-Aware Solution Embeddings}
\label{sec:emb}

To measure structural similarity between two solutions
without any additional learnable parameters, we reuse the graph
encoder's existing node representations
$\{\mathbf{h}_v\}_{v \in V}$ (output of the final attention layer)
and construct solution-level embeddings in a problem-adaptive manner.

\paragraph{TSP Problem.}
In TSP, all solutions visit the same node set; what differs is
\emph{which edges are traversed}.
For a solution route $\mathbf{s}^{(k)} = (s_0, \ldots, s_{T-1})$, with the closing convention $s_T^{(k)} := s_0^{(k)}$,
we compute the \emph{Hadamard product} for each consecutive edge:
\begin{equation}
  \mathbf{e}_t^{(k)} = \mathbf{h}_{s_t^{(k)}} \odot
  \mathbf{h}_{s_{t+1}^{(k)}}, \quad t = 0, \ldots, T-1,
  \label{eq:hadamard}
\end{equation}
and mean-pool over all edges:
\begin{equation}
  \mathbf{z}_k = \frac{1}{T}\sum_{t=0}^{T-1}
  \mathbf{h}_{s_t^{(k)}} \odot \mathbf{h}_{s_{t+1}^{(k)}}.
  \label{eq:hadamard_emb}
\end{equation}

\textbf{Properties of the Hadamard embedding.}
Three properties make the Hadamard product well suited to edge-level
structural encoding:
\begin{itemize}
  \item \emph{Symmetry}: $\mathbf{h}_i \odot \mathbf{h}_j =
    \mathbf{h}_j \odot \mathbf{h}_i$, correctly modeling
    undirected edges in TSP.
  \item \emph{Non-degeneracy}: Unlike simple addition,
    $\frac{1}{T}\sum_{t=0}^{T-1}(\mathbf{h}_{s_t}+\mathbf{h}_{s_{t+1}})
    = \frac{2}{T}\sum_{v\in V}\mathbf{h}_v$ for all-node closed TSP tours,
    yielding \emph{identical} embeddings for any route.
    Hadamard products over distinct node pairs yield irreducible,
    route-specific representations.
  \item \emph{Context-richness}: $\mathbf{h}_v$ is refined through
    $L$ mixed-attention layers integrating global topology, so
    $\mathbf{h}_i \odot \mathbf{h}_j$ encodes the
    \emph{interaction} between nodes $i$ and $j$ in the context of
    the full instance---far richer than raw coordinate differences.
\end{itemize}

\paragraph{EFL and JSP Problems.} In EFL, solution diversity is determined by which facility nodes are
selected; in JSP, by which operations are assigned to which machines.
We directly mean-pool the encoder embeddings of selected decision
nodes:
\begin{equation}
  \mathbf{z}_k = \frac{1}{T}\sum_{t=0}^{T-1}
  \mathbf{h}_{s_t^{(k)}}.
  \label{eq:assign_emb}
\end{equation}
Both embeddings reuse existing encoder representations with zero
additional parameters.

\textcolor{red}{\paragraph{Centering to remove a shared global direction.}}
Mean-pooling introduces a \emph{global base direction} shared by all
$B$ solutions (since every solution draws from the same encoder
embeddings of the same instance).
Without correction, all pairwise cosine similarities cluster near
$1.0$, making solutions appear structurally identical and collapsing
the weight distribution to near-uniform---effectively collapsing SSPO
to a structure-blind uniform baseline.
We therefore remove the shared direction by centering:
\begin{equation}
  \tilde{\mathbf{z}}_k = \mathbf{z}_k -
  \frac{1}{B}\sum_{j=1}^{B} \mathbf{z}_j.
  \label{eq:centering}
\end{equation}
After centering, cosine similarity purely captures
\emph{relative} structural differences among solutions, yielding a
discriminative and numerically stable measure.

\subsection{The Dissimilarity-Weighted Leave-One-Out Baseline}
\label{sec:baseline}

We now turn the centered embeddings into baseline weights. We compute pairwise cosine similarities on the centered embeddings:
\begin{equation}
  S_{ij} = \frac{\tilde{\mathbf{z}}_i^\top \tilde{\mathbf{z}}_j}
  {\|\tilde{\mathbf{z}}_i\|\,\|\tilde{\mathbf{z}}_j\|},
  \quad S_{ij} \leftarrow \mathrm{clamp}(S_{ij},\;0,\;1).
  \label{eq:cosine}
\end{equation}
The SSPO \emph{dissimilarity weight} for peer $j$ relative to
solution $i$ is:
\begin{equation}
  w_{ij} = \frac{1 - S_{ij}}{\displaystyle\sum_{k \neq i}(1 - S_{ik})},
  \quad w_{ii} = 0.
  \label{eq:weight}
\end{equation}
Intuitively, a peer that is \emph{structurally different} from
$\mathbf{s}_i$ represents a genuinely alternative strategy; its cost
carries far more information about the value landscape than a
near-duplicate.
The SSPO baseline for solution $i$ is:
\begin{equation}
  b_i^{\text{SSPO}} = \sum_{j \neq i} w_{ij}\,c_j.
  \label{eq:sspo_baseline}
\end{equation}

\paragraph{Graceful degradation to uniform baseline.}
When all solutions are structurally identical ($S_{ij} = c$ for all
$j \neq i$), the numerators $(1 - S_{ij})$ are equal and
$w_{ij} = 1/(B-1)$, recovering uniform leave-one-out weighting exactly.
In practice, a small $\epsilon = 10^{-8}$ is added to the denominator
of Eq.~\eqref{eq:weight} for numerical stability.

The advantage for solution $i$ is:
\begin{equation}
  A_i = b_i^{\text{SSPO}} - c_i.
  \label{eq:advantage}
\end{equation}
The SSPO policy gradient loss is:
\begin{equation}
  \mathcal{L} = -\frac{1}{B}\sum_{i=1}^{B}
  \mathrm{sg}(A_i)
  \sum_{t=0}^{T-1}\log\pi_\theta\!\left(a_t^{(i)} \mid s_t^{(i)}\right)
  - \beta\,\bar{H}(\pi_\theta),
  \label{eq:loss}
\end{equation}
where $\mathrm{sg}(\cdot)$ is stop-gradient (the baseline does not
receive gradients through embeddings), and $\bar{H}(\pi_\theta)$ is
an optional entropy bonus weighted by $\beta \geq 0$ to encourage
exploration.

The full procedure is summarized in Algorithm~\ref{alg:sspo}.

\begin{algorithm}[t]
\caption{SSPO Training}
\label{alg:sspo}
\begin{algorithmic}[1]
  \REQUIRE Policy $\pi_\theta$, group size $B$, entropy weight $\beta$
  \REPEAT
    \STATE Sample a mini-batch of instances $\{\mathcal{G}^{(m)}\}$
    \FOR{each instance $\mathcal{G}^{(m)}$}
      \STATE Sample $B$ solutions; compute costs $\{c_k\}$
      \STATE Obtain node embeddings $\{\mathbf{h}_v\}$ from encoder
      \STATE Build $\{\mathbf{z}_k\}$ via
        Eq.~\eqref{eq:hadamard_emb} or \eqref{eq:assign_emb}
      \STATE Center: $\tilde{\mathbf{z}}_k$ via Eq.~\eqref{eq:centering}
      \STATE Compute $S_{ij}$ via Eq.~\eqref{eq:cosine}
      \STATE Compute $w_{ij}$ via Eq.~\eqref{eq:weight}
      \STATE Compute $b_i^{\text{SSPO}}$ via Eq.~\eqref{eq:sspo_baseline}
    \ENDFOR
    \STATE Update $\theta$ with $\mathcal{L}$ via Eq.~\eqref{eq:loss}
  \UNTIL{convergence}
\end{algorithmic}
\end{algorithm}

\subsection{Complexity Analysis}

Beyond the standard REINFORCE forward pass, SSPO adds two operations:
(i) solution-level embedding construction via Hadamard products (TSP)
or mean-pooling (EFL/JSP), with cost $\mathcal{O}(B N D)$; and
(ii) a $B \times B$ cosine-similarity matrix, with cost
$\mathcal{O}(B^2 D)$.
For the settings used throughout our experiments, both are negligible compared to
the encoder--decoder forward pass ($\mathcal{O}(N^2 D L)$ per rollout,
executed $B$ times per instance).
SSPO introduces no additional learnable parameters beyond those of
the underlying policy network.

\subsection{Variance Trade-off Relative to RLOO}\label{sec:var_sketch}
The introduction identified \emph{baseline redundancy} as a
principal failure mode: when co-sampled peers are strongly correlated,
a uniformly weighted leave-one-out baseline loses variance-reduction
power. We now make this intuition precise, and characterize when
SSPO's dissimilarity weighting yields a  lower-variance
baseline than uniform RLOO.
Fix one training instance and one
reference rollout $i$. Let
\begin{equation*}
    \mathcal{J}_i
    =
    \{1,\ldots,B\}\setminus\{i\},
    \qquad
    m=B-1.
    \label{eq:peer_index_set}
\end{equation*}
Consider the general convex leave-one-out baseline
\begin{equation*}
    b_i(w)
    =
    \sum_{j\in\mathcal{J}_i} w_j c_j,
    \qquad
    w_j\geq 0,
    \qquad
    \sum_{j\in\mathcal{J}_i} w_j=1.
    \label{eq:general_weighted_loo}
\end{equation*}
Uniform RLOO is obtained by setting $w_j=1/m$:
\begin{equation*}
    b_i^{\mathrm{RLOO}}
    =
    \frac{1}{m}
    \sum_{j\in\mathcal{J}_i} c_j.
    \label{eq:rloo_baseline_variance_section}
\end{equation*}

To obtain a closed-form comparison, we adopt a standard
exchangeability-style assumption on peer costs. It is a mild
symmetry condition: peers drawn from the same policy on the same
instance share a marginal variance, and only their pairwise
correlations $\rho_{jk}$ vary---exactly the quantities SSPO's
structure-aware weights aim to exploit.

\begin{assumption}
\label{assumption:var_cov}
Peer costs have equal marginal variance and pairwise
covariances of the form
\begin{equation}
    \operatorname{Var}(c_j)=\sigma^2,
    \qquad
    \operatorname{Cov}(c_j,c_k)
    =
    \sigma^2\rho_{jk},
    \quad j\neq k.
    \label{eq:equal_marginal_assumption}
\end{equation}
\end{assumption}
Under Assumption~\ref{assumption:var_cov}, we obtain a closed-form expression for the variance gap introduced by the
SSPO weights, decomposed into a marginal-variance penalty and a signed
covariance-term reduction.
\begin{proposition}
\label{prop:exact_variance_tradeoff}
Variance
difference between the weighted leave-one-out baseline and
uniform RLOO is
\begin{align}
    &\operatorname{Var}\!\left(b_i(w)\right)
    -
    \operatorname{Var}\!\left(
        b_i^{\mathrm{RLOO}}
    \right)
    \nonumber\\
    = & 
    \sigma^2
    \Bigg[
        \underbrace{
            \sum_{j\in\mathcal{J}_i} w_j^2
            -
            \frac{1}{m}
        }_{\text{marginal-variance penalty}~\ge~0}
        -
        \underbrace{
            \sum_{\substack{j,k\in\mathcal{J}_i\\j\neq k}}
            \left(
                \frac{1}{m^2}
                -
                w_jw_k
            \right)
            \rho_{jk}
        }_{\text{signed covariance-term change}}
    \Bigg].
    \label{eq:variance_gap}
\end{align}
Consequently,
\begin{equation*}
    \operatorname{Var}\!\left(b_i(w)\right)
    \leq
    \operatorname{Var}\!\left(
        b_i^{\mathrm{RLOO}}
    \right)
    \label{eq:variance_dominance}
\end{equation*}
if and only if
\begin{equation*}
    \sum_{\substack{j,k\in\mathcal{J}_i\\j\neq k}}
    \left(
        \frac{1}{m^2}
        -
        w_jw_k
    \right)
    \rho_{jk}
    \geq
    \sum_{j\in\mathcal{J}_i}w_j^2
    -
    \frac{1}{m}.
    \label{eq:variance_reduction_condition}
\end{equation*}
\end{proposition}

\begin{proof}
For a fixed weight vector $w$, the variance of the weighted
baseline is
\begin{align*}
    \operatorname{Var}\!\left(b_i(w)\right)
    &=
    \sum_{j\in\mathcal{J}_i}
    w_j^2\operatorname{Var}(c_j)+
    \sum_{\substack{j,k\in\mathcal{J}_i\\j\neq k}}
    w_jw_k
    \operatorname{Cov}(c_j,c_k).
    \label{eq:weighted_baseline_variance_expansion}
\end{align*}
Substituting 
\begin{equation*}
    \operatorname{Var}(c_j)=\sigma^2,
    \qquad
    \operatorname{Cov}(c_j,c_k)
    =
    \sigma^2\rho_{jk},
    \quad j\neq k.
    \label{eq:equal_marginal_assumption}
\end{equation*}
gives
\begin{equation*}
    \operatorname{Var}\!\left(b_i(w)\right)
    =
    \sigma^2
    \left(
        \sum_{j\in\mathcal{J}_i}w_j^2
        +
        \sum_{\substack{j,k\in\mathcal{J}_i\\j\neq k}}
        w_jw_k\rho_{jk}
    \right).
    \label{eq:weighted_baseline_variance}
\end{equation*}
For uniform RLOO, $w_j=1/m$.
Then we have
\begin{align*}
    &\operatorname{Var}\!\left(b_i(w)\right)
    -
    \operatorname{Var}\!\left(
        b_i^{\mathrm{RLOO}}
    \right)
    \nonumber\\
    &=
    \sigma^2
    \Bigg[
        \sum_{j\in\mathcal{J}_i}w_j^2
        -
        \frac{1}{m}
        -
        \sum_{\substack{j,k\in\mathcal{J}_i\\j\neq k}}
        \left(
            \frac{1}{m^2}-w_jw_k
        \right)
        \rho_{jk}
    \Bigg].
    \label{eq:variance_gap_appendix}
\end{align*}
\end{proof}

Uniform RLOO minimizes $\sum_j w_j^2$ over the probability
simplex. Therefore, the first term in~\eqref{eq:variance_gap} is always nonnegative: any deviation from uniform weights incurs a
positive marginal-variance penalty. 
The second term is an aggregate covariance adjustment and can have either sign. For a positively correlated peer pair satisfying $w_jw_k<1/m^2$, the quantity $(1/m^2-w_jw_k)\rho_{jk}$ is positive, and contributes to reducing the variance difference. The benefit becomes larger as the correlation $\rho_{jk}$ increases and as $w_jw_k$ falls further below $1/m^2$. Therefore SSPO achieves a lower baseline variance when the resulting aggregate covariance reduction is sufficient to offset the marginal-variance penalty induced by nonuniform weights.

Hence SSPO has lower baseline variance precisely when its structure-aware
weights sufficiently reduce the contribution of highly
correlated peer pairs to offset the penalty caused by weight
concentration---the regime our experiments (Section~\ref{sec:analysis}) identify as structurally diverse.

\section{Experiments}

\subsection{Experimental Setup}

We evaluate SSPO on three problem families chosen to probe
different assumptions in preference-optimization NCO training.
\emph{TSP} ($N \in \{50, 100\}$, uniform $[0,1]^2$; gap vs.\
LKH-3~\cite{helsgaun2017extension}) is edge-structured, with solutions
that share most edges early in training---the regime where baseline
redundancy is most severe.
\emph{EFL} (Entity Facility Location) is a real-world facility
placement benchmark with coverage and demand constraints; its solution
space is assignment-dominated, so co-sampled solutions can occupy
completely disjoint regions and best-anchor methods risk locking onto
a single suboptimal mode.
\emph{JSP} (Job-Shop Scheduling) spans three standard families
(LA/TA/DMU) whose difficulty and instance size vary widely, letting
us probe how each method scales with problem hardness.

\textbf{Fair comparison protocol.} For every problem--method pair, all
methods share the identical backbone encoder ($L=6$ mixed-attention
layers, 8 heads, hidden dimension $D=256$), the same random
initialization, the same group size $B=256$, and the same optimizer
schedule. The only difference across methods is the baseline
formulation. This isolates the effect of the baseline design---the
central object of study in this paper---from confounders such as
architecture width, sampling strategy, or curriculum. Optimality gaps
are reported against LKH-3 (TSP), the best result among all compared
methods (EFL, for which exact optimization is computationally intractable), and
best-known solutions (JSP).

\textbf{Baselines and ablation.} We compare SSPO against three
categories of prior work: (i) \textbf{BOPO}~\cite{liao2025bopo}, the current
state-of-the-art best-anchor preference-optimization method for NCO;
(ii) \textbf{POMO}~\cite{kwon2020pomo} and \textbf{Sym-NCO}~\citep{kim2022symnco} for TSP,
and \textbf{SLIM}~\cite{NEURIPS2024_be8987a7}  for JSP, as representative structure-blind
baselines; and (iii) \textbf{uniform RLOO}, which shares SSPO's
leave-one-out form but sets $w_{ij}=1/(B-1)$. RLOO is critical for
our analysis: comparing SSPO with RLOO isolates the contribution of
\emph{structure-aware weighting}, while comparing RLOO with BOPO
isolates the contribution of the \emph{leave-one-out form} itself.
Together the three-way comparison exposes the trade-off that
motivates SSPO's design (Section~\ref{sec:tradeoff}).

\subsection{Main Results}

Tables~\ref{tab:tsp_main}--\ref{tab:jsp_main} show that SSPO
attains the best average optimality gap on \emph{five out of six}
main-benchmark settings, and is within 0.05pp of the best result on the sixth.
Beyond the raw ranking, the tables reveal a consistent qualitative
pattern across the three problem families, which we discuss in turn.

\paragraph{TSP.}
On TSP, SSPO improves over BOPO at both scales, and the margin grows
with instance size: from 0.01pp on TSP-50 (0.14\% vs.\ 0.15\%) to
0.08pp on TSP-100 (0.40\% vs.\ 0.48\%).
This size dependence is the signature predicted by our variance
analysis (Section~\ref{sec:var_sketch}): larger instances admit more
topologically distinct co-sampled solutions, so the signed covariance-term reduction
in Proposition~\ref{prop:exact_variance_tradeoff} grows, and the
dissimilarity weighting has more leverage to offset the marginal-variance penalty.

\begin{table}[t]
\centering
\caption{TSP optimality gap (\%). Best in \textbf{bold}.}
\label{tab:tsp_main}
\begin{tabular}{lcc}
\toprule
Method & TSP-50 & TSP-100 \\
\midrule
POMO                       & 0.21\%          & 1.04\% \\
Sym-NCO                    & ---          & 0.94\% \\ 
BOPO                       & 0.15\%          & 0.48\% \\
RLOO                       & 0.33\%               & 1.32\% \\
\midrule
\textbf{SSPO (ours)}       & \textbf{0.14\%} & \textbf{0.40\%} \\
\midrule
 LKH-3                       & 0\%               & 0\% \\
\bottomrule
\end{tabular}
\end{table}

\paragraph{EFL.}
On EFL, we report the relative gap to the best result among all
compared methods. SSPO attains this best result (0\% relative gap),
while BOPO and POMO trail by 7.31\% and 7.34\% respectively. The gap here is
much larger than on TSP, and the cause is structural rather than
numerical: the assignment-dominated solution space is far wider than
an edge-structured tour, so a best-anchor method that concentrates
its learning signal on a single top solution tends to lock the policy
into a suboptimal mode. By scoring all $B$ solutions jointly and
weighting them by structural dissimilarity, SSPO keeps the whole
group informative throughout training.

\begin{table}[t]
\centering
\caption{EFL relative gap (\%) to the best result among all compared
methods. Best in \textbf{bold}.}
\label{tab:efl_main}
\begin{tabular}{lc}
\toprule
Method & Gap to Best \\
\midrule
POMO                 & 7.34\% \\
BOPO                 & 7.31\% \\
RLOO                 & 2.64\% \\
\midrule
\textbf{SSPO (ours)} & \textbf{0\%} \\
\bottomrule
\end{tabular}
\end{table}

We further probe SSPO on production-scale
facility-placement cases at JD.com, whose demand and coverage distributions are considerably less regular than the standard EFL benchmark. Figure~\ref{fig:efl_cases} shows two representative
real-world instances. In both, the SSPO policy places facilities
so that demand coverage stays consistent with the pattern reported on
the benchmark, confirming that the method transfers beyond the
regularized EFL setting.

\begin{figure*}[t]
\centering
\includegraphics[width=\linewidth]{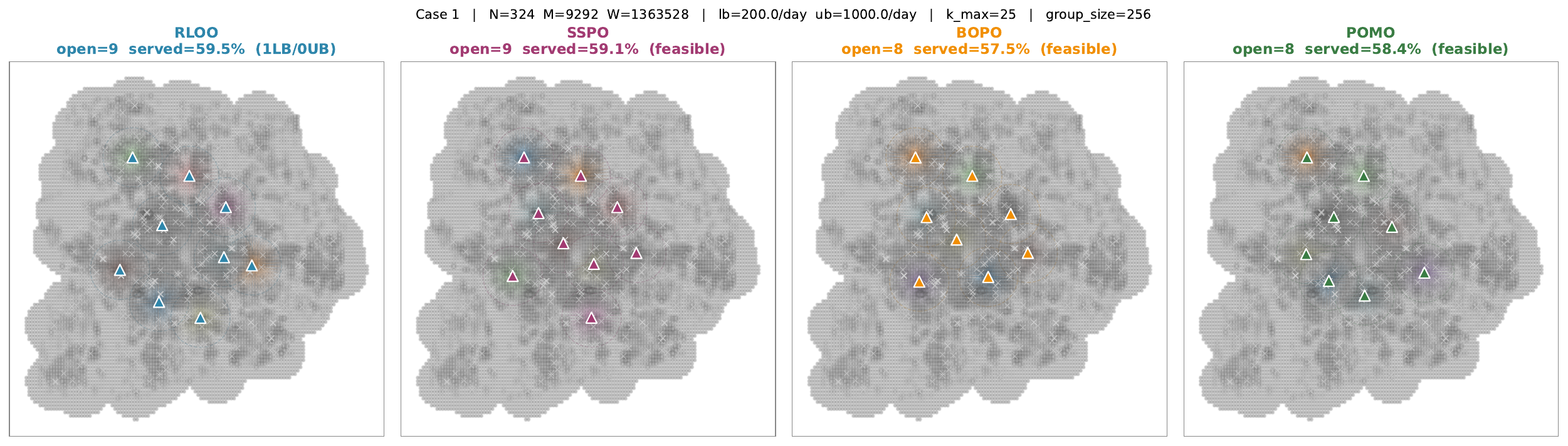}\\[4pt]
\includegraphics[width=\linewidth]{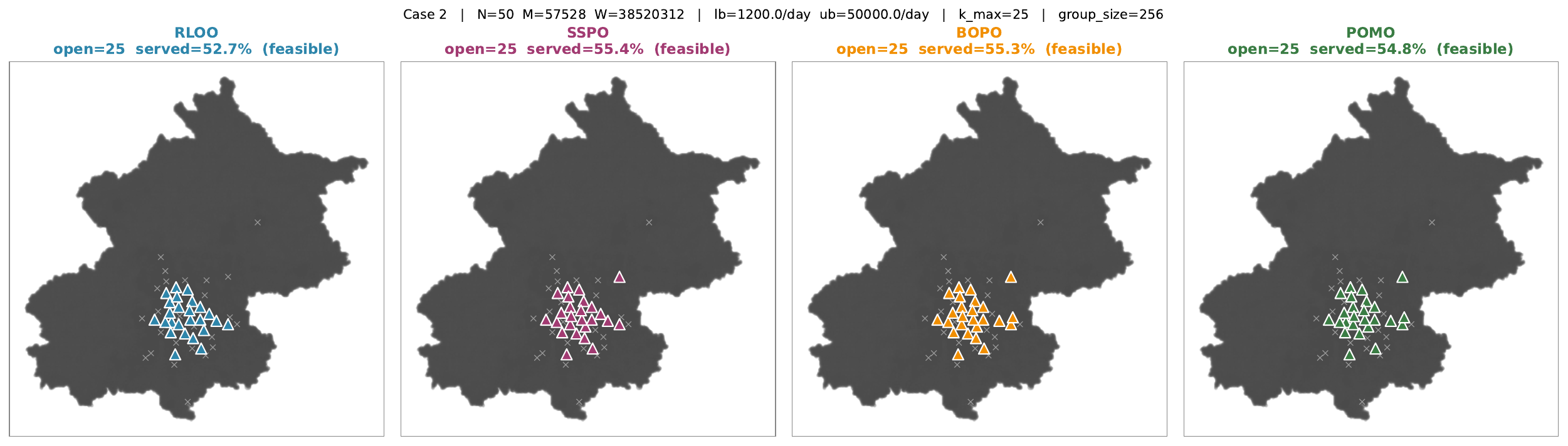}
\caption{Comparison of real-world EFL facility-placement solutions produced by RLOO, SSPO, BOPO, and POMO.
These production-scale cases have far less regular demand and coverage
distributions than the standard EFL benchmark, yet SSPO yields
placements consistent with the benchmark results.}
\label{fig:efl_cases}
\end{figure*}

\paragraph{JSP.}
On the harder TA and DMU families SSPO improves over BOPO by 1.14pp
(7.52\% vs.\ 8.66\%) and 0.86pp (12.95\% vs.\ 13.81\%). On the easier
LA family BOPO holds a marginal 0.05pp edge (2.50\% vs.\ 2.55\%).
This reversal is consistent with our analysis rather than adverse to
it: on compact instances the policy settles into a tight mode early,
so structural diversity among co-sampled peers is low. Consequently, the signed covariance-term reduction in Proposition 1 shrinks toward zero, and SSPO approaches uniform RLOO. The residual 0.05pp gap matches the
small marginal-variance penalty $\sum_j w_{ij}^2 - 1/(B-1)$ that SSPO
pays whenever its weights deviate from uniform.

\begin{table}[t]
\centering
\caption{JSP average optimality gap (\%) across benchmark families.
Best in \textbf{bold}.}
\label{tab:jsp_main}
\begin{tabular}{lccc}
\toprule
Method & LA Avg & TA Avg & DMU Avg \\
\midrule
SLIM                  & 2.59\% & 7.84\%          & 13.11\% \\
BOPO                 & \textbf{2.50\%} & 8.66\%          & 13.81\% \\
\midrule
\textbf{SSPO (ours)} & 2.55\%          & \textbf{7.52\%} & \textbf{12.95\%} \\
\bottomrule
\end{tabular}
\end{table}
\subsection{Analysis}
\label{sec:analysis}

\paragraph{Best-anchor vs.\ uniform-LOO trade-off.}
\label{sec:tradeoff}
The most informative comparison in Tables~\ref{tab:tsp_main} and
\ref{tab:efl_main} is not SSPO vs.\ BOPO, but BOPO vs.\ RLOO---the two
baseline designs that SSPO subsumes---because their ranking reverses
across problem families. On TSP-100 the best-anchor signal is strong:
BOPO (0.48\%) beats RLOO (1.32\%) by $2.75\times$, since co-sampled
tours share many edges and anchoring on the top tour concentrates the
gradient on a meaningful direction. On EFL the picture inverts: RLOO
(2.64\%) beats BOPO (7.31\%) by $2.8\times$, because the assignment
space is wide and anchoring on one solution strands the policy in a
local mode. Neither baseline is uniformly correct; each fixes one
failure mode and aggravates the other. SSPO is designed around this
trade-off---it keeps the leave-one-out form to avoid best-anchor
lock-in while re-weighting peers by structural dissimilarity to avoid
the redundant baseline of uniform RLOO---and it improves over the
better of the two in both regimes (0.40\% vs.\ 0.48\% on TSP-100;
0\% vs.\ 2.64\% on EFL).

\paragraph{Effect of structural diversity.}
\label{sec:experiments_when}
The size of SSPO's advantage tracks how structurally diverse the
co-sampled solutions are. Where diversity is high---EFL, and the
larger TSP and harder JSP instances---SSPO's gain over BOPO is
largest, matching Proposition~\ref{prop:exact_variance_tradeoff}: the signed covariance-term reduction $\sum_{j\neq k}(1/m^2 - w_jw_k)\rho_{jk}$ scales with peer dissimilarity, so the signal
SSPO recovers grows with the diversity that best-anchor methods miss.
Where diversity is low, as on the compact JSP LA family, the bound
tightens to equality and SSPO's margin shrinks to a marginal gap of 0.05pp.

\paragraph{Ablation: structure-aware weighting.}
Replacing SSPO's dissimilarity weights with uniform weights (i.e.,
recovering plain RLOO) on TSP-100 degrades the gap from 0.40\% to
1.32\%---a $3.3\times$ deterioration on the same architecture,
initialization, and rollouts. This isolates the source of the gain.
The leave-one-out form alone is insufficient: uniform RLOO
underperforms even BOPO (1.32\% vs.\ 0.48\%), so the improvement does
not come from excluding the current solution's own cost from the baseline. The gain
comes from the weighting itself---substituting uniform
$w_{ij}=1/(B-1)$ with the dissimilarity-driven $w_{ij}\propto 1-S_{ij}$,
with no other change, accounts for the improvement from 1.32\% to 0.40\%.
Together with the trade-off analysis above, this attributes SSPO's
improvement to structure-aware weighting as a design choice rather
than to the leave-one-out form or a tuned encoder, exactly as
Section~\ref{sec:baseline} anticipates.

\section{Conclusion}

We presented \textbf{SSPO}, which reweights the leave-one-out
baseline by structural dissimilarity between co-sampled solutions.
This single mechanism addresses two orthogonal failure modes in NCO
training---\emph{gradient signal polarization} and \emph{baseline
redundancy}---while remaining zero-parameter and preserving the
leave-one-out form.
Experiments on TSP, EFL, and JSP show consistent gains over
best-anchor and uniform-weight baselines, and a direct comparison
against uniform RLOO isolates structure-aware weighting as the
primary driver.
More broadly, our results suggest that the structural information
encoded in the graph representations---so far used only for action
selection---is a rich, underused signal for variance reduction,
opening the door to structure-aware training beyond the specific
baselines studied here.

\appendix
\bibliography{references}

\end{document}